\documentclass{ecai} 

\usepackage{latexsym}
\usepackage{amsmath}
\usepackage{xspace}
\usepackage{float}
\usepackage{amsfonts}
\usepackage{amssymb}
\usepackage{amsthm}
\usepackage{mathtools}
\usepackage{graphicx}
\usepackage{caption}
\usepackage{subfig}
\usepackage{tabularx}
\usepackage{multirow}
\usepackage{xcolor}
\usepackage{pifont}
\usepackage{bm}
\usepackage{arydshln}
\usepackage[linesnumbered,ruled,vlined]{algorithm2e}
\usepackage{todonotes}

\usepackage{paralist}

\newtheorem{theorem}{Theorem}
\newtheorem{lemma}[theorem]{Lemma}

\newtheorem{definition}{Definition}

\newcommand{\BibTeX}{B\kern-.05em{\sc i\kern-.025em b}\kern-.08em\TeX}

\newcommand{\myi}{(\emph{i})\xspace}
\newcommand{\myii}{(\emph{ii})\xspace}
\newcommand{\myiii}{(\emph{iii})\xspace}

 \newcommand{\D}{\mathcal{D}}
 \newcommand{\F}{\mathcal{F}}
\newcommand{\G}{\mathcal{G}}

\newcommand{\LL}{\mathcal{L}}
 
\renewcommand{\O}{\mathcal{O}}
\newcommand{\PP}{\mathcal{P}}
 \newcommand{\R}{\mathcal{R}}

\newcommand{\U}{\mathcal{U}}

 \newcommand{\X}{\mathcal{X}}
\newcommand{\Y}{\mathcal{Y}} 

\newcommand{\Wnext}{\raisebox{-0.27ex}{\LARGE$\bullet$}}
\newcommand{\Next}{\raisebox{-0.27ex}{\LARGE$\circ$}}
\newcommand{\Until}{\mathop{\U}}
\newcommand{\Release}{\mathop{\R}}

\newcommand{\trueVal}{\mathit{true}} 
\newcommand{\falseVal}{\mathit{false}} 
\newcommand{\ttrue}{\mathit{tt}} 
\newcommand{\ffalse}{\mathit{ff}} 

\newcommand{\LTL}{{\sc ltl}\xspace}

\newcommand{\LTLf}{{\sc ltl}$_f$\xspace}
\newcommand{\ltlf}{{\sc ltl}$_f$\xspace}

\newcommand{\DFA}{{\sc dfa}\xspace}
\newcommand{\DFAs}{{\sc dfa}s\xspace}

\newcommand{\dfa}{{\sc dfa}\xspace}

\newcommand{\Nat}{{\rm I\kern-.23em N}}

\newcommand{\roundBra}[1]{\left( #1 \right)}
\newcommand{\curlyBra}[1]{\left\{ #1 \right\}}

\def\BDD{\textsf{BDD}\xspace}

\def\tople{{\sf Tople}\xspace}

\def\nike{{\sf Nike}\xspace}
\def\lisa{{\sf Lisa}\xspace}
\def\lydia{{\sf LydiaSyft}\xspace}

\def\tool{{\sf Cosy}\xspace}

\def\awin{\textnormal{\sf\small awin}\xspace}
\def\awr{\textnormal{\sf\small awr}\xspace}

\newcommand{\aequiv}{\mathop{\equiv_a}}

\SetKwProg{myproc}{function}{}{}
\SetKwRepeat{Do}{do}{while}

\SetKw{Break}{break}
\SetKw{Continue}{continue}
\SetKw{Init}{initialize}
\SetKw{Stack}{stack}

\SetKwFunction{add}{add}
\SetKwFunction{remove}{remove}
\SetKwFunction{GroupByCommonSysVar}{GroupByCommonSysVar}
\SetKwFunction{SolveInterferingPart}{SolveInterferingPart}
\SetKwFunction{IsRealizable}{IsRealizable}
\SetKwFunction{GetSwr}{GetSwr}
\SetKwFunction{GetAwr}{GetAwr}
\SetKwFunction{BuildAwr}{BuildAwr}
\SetKwFunction{SolveDfaGame}{SolveDfaGame}
\SetKwFunction{IncreComposeOnce}{IncreComposeOnce}
\SetKwFunction{DFSearch}{DFSearch}
\SetKwFunction{ReadSuccessor}{ReadSuccessor}
\SetKwFunction{IsAccepting}{IsAccepting}
\SetKwFunction{GetSuccessor}{GetSuccessor}
\SetKwFunction{ComposeIncrementally}{ComposeIncrementally}
\SetKwFunction{ComposeIndividually}{ComposeIndividually}
\SetKwFunction{Compose}{Compose}
\SetKwFunction{BuildStrategy}{BuildStrategy}

\SetKwFunction{EwinAgentChoices}{EwinAgentChoices}
\SetKwFunction{IsSccRoot}{IsSccRoot}
\SetKwFunction{GetScc}{GetScc}
\SetKwFunction{ForwardSearch}{ForwardSearch}
\SetKwFunction{BackwardSearch}{BackwardSearch}
\SetKwFunction{GetEdge}{GetEdge}
\SetKwFunction{CheckY}{CheckY}
\SetKwFunction{GetUnserchedX}{GetUnserchedX}
\SetKwFunction{BlockYFromTrancducer}{BlockYFromTrancducer}
\SetKwFunction{AddWeakSwinTransition}{AddWeakSwinTransition}

\begin{document}


\begin{frontmatter}


\paperid{7880} 


\title{A Compositional Framework for \\ On-the-Fly \ltlf Synthesis}


\author[A]{\fnms{Yongkang}~\snm{Li}} 
\author[A]{\fnms{Shengping}~\snm{Xiao}} 
\author[B]{\fnms{Shufang}~\snm{Zhu}} 
\author[A]{\fnms{Jianwen}~\snm{Li}\thanks{Corresponding Author. Email: lijwen2748@gmail.com}}
\author[A]{\fnms{Geguang}~\snm{Pu}}

\address[A]{East China Normal University, Shanghai, China}
\address[B]{University of Liverpool, Liverpool, UK}



\begin{abstract}
Reactive synthesis from Linear Temporal Logic over finite traces (\ltlf) can be reduced to a two-player game over a Deterministic Finite Automaton~(\DFA) of the \ltlf specification. The primary challenge here is \DFA construction, which is 2EXPTIME-complete in the worst case. Existing techniques either construct the \DFA compositionally before solving the game, leveraging automata minimization to mitigate state-space explosion, or build the \DFA incrementally during game solving to avoid full \DFA construction. However, neither is dominant. In this paper, we introduce a compositional on-the-fly synthesis framework that integrates the strengths of both approaches, focusing on large conjunctions of smaller \ltlf formulas common in practice. This framework applies composition during game solving instead of automata~(game arena) construction. While composing all intermediate results may be necessary in the worst case, pruning these results simplifies subsequent compositions and enables early detection of unrealizability. Specifically, the framework allows two composition variants: pruning before composition to take full advantage of minimization or pruning during composition to guide on-the-fly synthesis. Compared to state-of-the-art synthesis solvers, our framework is able to solve a notable number of instances that other solvers cannot handle. A detailed analysis shows that both composition variants have unique merits.

\end{abstract}

\end{frontmatter}


\section{Introduction}

A key challenge in Artificial Intelligence~(AI) is enabling intelligent agents to autonomously plan and execute complex actions to achieve desired tasks~\citep{Reiter2001,GhallabNauTraverso2016}. This challenge aligns with reactive synthesis in Formal Methods, where an agent operates in an adversarial environment, controlling certain variables while the environment controls others. Given a task specification, the agent must devise a strategy to achieve the task despite possible environmental reactions~\citep{Pnu85}. Reactive synthesis also shares deep similarities with planning in fully observable nondeterministic domains (FOND, strong plans)~\citep{Cimatti03,GeffnerBonet2013}.

In Formal Methods, tasks are typically specified using Linear Temporal Logic (\LTL)~\citep{Pnu77}. In AI, a finite trace variant of \LTL, i.e., \LTLf~\citep{GV13}, is popular, reflecting the fact that intelligent agents usually handle tasks one after another rather than dedicating to a single task all their lifetime. Apart from reactive synthesis, \LTLf has been extensively applied in various domains such as automated planning with temporal goals and trajectory constraints~\citep{pddl3,camacho2019strong}, rewards in reinforcement learning~\citep{GLMF19}, and reasoning about business processes~\citep{PBV10,AGCFJ20}. In this work, we focus on \LTLf synthesis~\citep{GV15}.

\LTLf synthesis can be solved by reducing it to an adversarial reachability game on a Deterministic Finite Automaton (\DFA) that recognizes the same language as the \LTLf formula~\citep{GV15}. This process involves two parts: \DFA construction and \DFA game solving. While the \DFA game can be solved in linear time~\citep{Gradel07}, the \DFA itself can be, in the worst case, double-exponential in the size of the formula~\citep{GV13}, making its construction the bottleneck in \ltlf synthesis~\citep{ZTLPV17}.

Existing approaches to \LTLf synthesis address this challenge through two common directions: backward or forward. The backward approach constructs a complete \DFA in a compositional manner, applying \DFA minimization at each composition step to mitigate state space explosion~\citep{BLTV20,DF21} and solves the \DFA game via backward search using efficient symbolic techniques~\citep{ZTLPV17}. The forward approach constructs the \dfa on the fly, starting from the initial state and progressively building the \dfa while simultaneously solving the \dfa game~\citep{XLZSPV21,GFLVXZ22,XLHXLPSV24,XLZSLPV24}, to possibly avoid the double-exponential blowup.
However, neither approach is dominant while both demonstrate strengths as well as inefficiency on certain benchmarks~\citep{XLZSLPV24}. 

A natural question arises: \emph{can we integrate the strengths of both \DFA minimization and on-the-fly synthesis to maximize their advantages?} We demonstrate both theoretically and experimentally that \LTLf specifications in the form of large conjunctions of smaller \LTLf formulas offer a good compromise between expressiveness and synthesis efficiency.



We introduce a compositional framework for on-the-fly \ltlf synthesis.
Within this framework, we first decompose the original synthesis problem, i.e., $\varphi = \varphi_1 \wedge \varphi_2 \wedge \cdots \wedge \varphi_n$, w.r.t. conjunctions into sub-problems, i.e., $\varphi_1, \varphi_2,  \cdots, \varphi_n$. Then we process each $\varphi_i$ utilizing forward synthesis and compose the intermediate results to solve the original problem on $\varphi$. In this case, we conduct composition at the synthesis~(game solving) level, rather than during \DFA construction. Hence, leveraging the conjunction form of $\varphi$, we can immediately conclude $\varphi$ being unrealizable as soon as some part of $\varphi$ is found unrealizable, without further composition.

The crux of the compositional framework lies in efficiently composing intermediate results from sub-problems to speed up subsequent synthesis. The efficiency of forward \LTLf synthesis heavily relies on the size of the search space and the search direction. To reduce the search space, we can utilize \DFA minimization. Suppose we have two sub-problems $\varphi_i$ and $\varphi_j$. To synthesize their conjunction $\varphi_i \wedge \varphi_j$, we can synthesize them separately, build their respective \emph{winning regions} as \DFAs to ensure their realization under all possible environment behaviours, minimize these \DFAs, compose them and minimize again to obtain the ultimate search space for $\varphi_i \wedge \varphi_j$. This composition is referred to as \emph{individual composition}. However, this method may produce states that are irrelevant to realize $\varphi_i \wedge \varphi_j$. Specifically, if $\varphi_j$ is relatively large, constructing its winning region could be inefficient. To address this, we propose an \emph{incremental composition} variant that utilizes the previously computed winning region of $\varphi_i$ to guide the search of the second sub-problem $\varphi_j$, performing an implicit conjunction during the search and directly obtaining the synthesis result for $\varphi_i \wedge \varphi_j$.

We implemented our synthesis framework in a prototype tool called \tool. To evaluate the efficiency of the proposed composition techniques, we conducted an empirical evaluation by comparing \tool with state-of-the-art \LTLf synthesizers. Experimental results show that our compositional synthesis framework solves a notable number of instances that other synthesizers cannot handle, demonstrating that it effectively combines the merits of both backward and forward synthesis approaches for \LTLf specifications. A detailed comparison of the two composition variants shows that, while the incremental variant solves slightly more instances, both variants have their strengths and advantages. The full version of the paper including appendix is available in~\cite{li2025compositionalframeworkontheflyltlf}.

\section{Preliminaries}



\subsection{\LTLf Basics}
Linear Temporal Logic over finite traces, or \ltlf~\citep{GV13}, extends propositional logic with finite-horizon temporal connectives. \LTLf is a semantics variant of Linear Temporal Logic~(\LTL)~\citep{Pnu77} adapted for finite traces. Given a set of atomic propositions $\PP$, the syntax of \ltlf is identical to \LTL, and defined as: $\varphi ::= \ttrue \mid p \mid \neg \varphi \mid \varphi \wedge \varphi \mid \Next \varphi \mid \varphi \Until \varphi$,
where $\ttrue$ denotes the \emph{true} formula, $p \in \PP$ is an atomic proposition, $\neg$ denotes \emph{negation}, $\wedge$ is \emph{conjunction}, $\Next$ is the \emph{strong Next} operator, and $\Until$ is the \emph{Until} operator. We have their corresponding dual operators: $\ffalse$~(\emph{false}) for $\ttrue$, $\vee$~(\emph{disjunction}) for $\wedge$, $\Wnext$~(\emph{weak Next}) for $\Next$, and $\Release$~(\emph{Release}) for $\Until$. Moreover, we use $\G\varphi$ (\emph{Global}) and $\F\varphi$ (\emph{Eventually}) to represent $\ffalse \Release\varphi$ and $\ttrue \Until\varphi$, respectively. 
The length of $\varphi$ is denoted by $|cl(\varphi)|$, where $cl(\varphi)$ is the set of all sub-formulas of $\varphi$.

A \emph{finite non-empty trace} $\rho = \rho[0]\; \rho[1] \cdots \rho[n]\in(2^\PP)^+$ is a sequence of propositional interpretations, where $\rho[i]$ represents the set of propositions that are $\trueVal$ at instant $i$. 
\ltlf formulas are interpreted over finite non-empty traces. Hereafter, we use the term \emph{trace} to refer to a finite non-empty trace for simplicity. Given a trace $\rho$ of length $n+1$, and an instant $0 \leq i \leq n$, we denote by $\rho_i$ the suffix of $\rho$ starting at $i$, i.e., $\rho_i = \rho[i]\; \rho[i+1] \cdots \rho[n]$.
For a trace $\rho$ and an \ltlf formula $\varphi$, we define $\rho$ satisfies $\varphi$, denoted $\rho\models\varphi$, as follows:
\begin{compactitem}
	\item 
	$\rho \models \ttrue$;
	\item 
	$\rho \models p$ iff $p \in \rho[0]$, 
        where $p\in\PP$ is an atomic proposition;
        \item 
        $\rho \models \neg\varphi$ iff $\rho\not\models\varphi$;
	\item 
	$\rho \models \varphi_1 \wedge \varphi_2$ iff $\rho \models \varphi_1$ and $\rho \models \varphi_2$;
	\item 
	$\rho \models \Next \varphi$ iff $|\rho|>1$ and $\rho_1 \models \varphi$;
    \item 
	$\rho \models \varphi_1 \Until \varphi_2$ iff there exists $i$ with $0\leq i < |\rho|$ such that $\rho_i\models \varphi_2$ holds, and for every $j$ with $0 \leq j < i$ it holds that $\rho_j \models \varphi_1$.
\end{compactitem}
The set of traces that satisfy \LTLf formula $\varphi$ is the language of $\varphi$, denoted as $\LL(\varphi)=\{\rho\in(2^\PP)^+\mid \rho\models\varphi\}$.

\subsection{\ltlf Synthesis}

An \LTLf synthesis specification is a tuple $(\varphi,\X, \Y)_t$, where $\varphi$ is an \LTLf formula over propositions in $\X\cup\Y$, with $\X$ being the set of input variables controlled by the environment, $\Y$ being the set of output variables controlled by the agent, and $\X\cap\Y = \emptyset$. The parameter $t\in\curlyBra{Mealy, Moore}$ is the type of target reactive system~(strategy/policy/plan). In reactive systems, interactions happen in turns where both the agent and the environment make moves by assigning values to their respective controlled variables. The order of assignment within each turn determines the system types: if the environment moves first, the system is a Mealy machine; conversely, if the agent moves first, it is a Moore machine. For brevity, we only focus on the problem of synthesizing Moore machines and omit the parameter $t$ hereafter.

\begin{definition}[Winning Strategy]\label{def:winning}
Given an \LTLf synthesis specification $\roundBra{\varphi,\X,\Y}$, 
an agent strategy $g : (2^{\mathcal{X}})^{*} \to 2^{\mathcal{Y}}$ is a \emph{winning strategy} for $\varphi$ iff 
for every infinite sequence $\lambda = X_0 X_1 \cdots \in (2^\X)^\omega$ of propositional interpretations over $\X$, i.e., every possible environment behaviours, there is $k \geq 0$ such that $\rho\models\varphi$ holds, where $\rho=(X_0\cup g(\epsilon)) (X_1\cup g(X_0)) \cdots (X_k\cup g(X_0 \cdots X_{k-1}))$
\end{definition}

\begin{definition}[\LTLf Realizability]\label{def:realizability}
Given an \LTLf synthesis specification $\roundBra{\varphi,\X,\Y}$, 
it is \emph{realizable} iff there exists a winning strategy.
An \LTLf specification is \emph{unrealizable} if it is not realizable.
The \emph{\LTLf realizability problem} is to determine whether an \ltlf specification 
$\roundBra{\varphi,\X,\Y}$ is realizable.
\end{definition}

\begin{definition}[\LTLf Synthesis]\label{def:synthesis}
The \emph{\LTLf synthesis problem} (for a realizable specification) is to compute a winning strategy.
\end{definition}

\subsection{\dfa, \dfa Product, and \dfa Games}

A Deterministic Finite Automaton (\DFA) is described as a 5-tuple $\G=(2^\PP,S,init,\delta,F)$, where
\begin{compactitem}
    \item $2^\PP$ is the alphabet;
    \item $S$ is a finite set of states;
    \item $init\in S$ is the initial state;
    \item $\delta:S\times 2^\PP\to S$ is the transition function;
    \item $F\subseteq S$ is a set of accepting states.
\end{compactitem}

The run $r$ of a trace $\rho=\rho[0]  \rho[1] \cdots \rho[n] \in (2^\PP)^+$ on a \DFA $\G$ is a finite sequence of states $r = s_0 s_1 \cdots s_{n+1}$ such that $s_0=init$ and $\delta(s_i,\rho[i])=s_{i+1}$ for $0\leq i \leq n$.
A trace $\rho$ is accepted by $\G$ iff the corresponding run $r$ ends with an accepting state (i.e., $s_{n+1}\in F$). The set of finite traces accepted by a \DFA $\G$ is the language of $\G$, denoted as $\LL(\G)$. Given any \DFA, we can construct an equivalent \DFA that recognizes the same language and has the minimum number of states. This minimization process can be performed in $\O(|S|\;log\;|S|)$ time using Hopcroft's algorithm~\citep{Hop71}, or in $\O(|S|^2)$ time using Moore's algorithm~\citep{Moo56}.

Given two \DFAs $\G_1=(2^\PP,S_1,init_1,\delta_1,F_1)$ and $\G_2=(2^\PP,S_2,init_2,\delta_2,F_2)$, the product \DFA $\G_1\times\G_2=(2^\PP,S,init,\delta,F)$ such that $\LL(\G_1\times\G_2)=\LL(\G_1)\cap\LL(\G_2)$ is constructed as follows:
\begin{compactitem}

\item $S=S_1\times S_2$ is the set of states;

\item $init = (init_1,init_2)\in S$ is the initial state;

\item $\delta$ is the transition function such that $\delta((s_1,s_2),\sigma) =(\delta_1(s_1,\sigma),\delta_2(s_2,\sigma))$ , where $s_1\in S_1$, $s_2\in S_2$, and $\sigma\in2^\PP$;

\item $F=\{(s_1,s_2)\in S_1\times S_2\mid s_1\in F_1\text{ and }s_2\in F_2\}$ is the set of accepting states.
\end{compactitem}

For every \ltlf formula $\varphi$ over $\mathcal{P}$, there exists a \dfa $\G_\varphi$ that recognizes the same language as $\varphi$, i.e., $\LL(\varphi)=\LL(\G_\varphi)$~\citep{GV13}. There are various approaches to \LTLf-to-DFA construction. In this paper, we leverage the on-the-fly \LTLf-to-DFA construction technique presented in~\citep{GFLVXZ22,XLZSLPV24}. In this approach, every state of the DFA $\G_\varphi$ is represented by an \ltlf formula $\psi$, derived from $\varphi$. The initial state is the formula $\varphi$ itself, and the successor states are generated using $\GetSuccessor(\psi,\sigma)$, where $\sigma \in 2^\mathcal{P}$. Intuitively,  $\GetSuccessor(\psi,\sigma)$ computes the successor of the state $\psi$ via the transition condition $\sigma$. Additionally, we use $\IsAccepting(\psi)$ to check whether $\psi$ is an accepting state. The upper bound of the constructed \DFA size is $\O(2^{2^{|cl(\varphi)|}})$. For more details, we refer to~\citep{GFLVXZ22}.

\begin{lemma}\label{lem:and2product}
Given an \ltlf formula $\varphi = \bigwedge_{1\leq i\leq n}\varphi_i$ and \dfa $\G_1,\cdots,\G_n$ that recognize the language of conjuncts $\varphi_1,\cdots,\varphi_n$, we have $\LL(\varphi)=\LL(\G_1\times\cdots\times\G_n)$.
\end{lemma}

A \dfa game~\citep{GV15} is an adversarial reachability game between two players, the agent and the environment, where the \dfa serves as the game arena. The agent and the environment control two disjoint sets of variables $\Y$ and $\X$, respectively. Starting from the initial state, each round consists of both players making moves by assigning values to the variables within their control. The subsequent state is determined following the \DFA transition function, producing a play of the game. A play is the sequence of states produced during the interaction between the agent and the environment. A play terminates when it reaches an accepting state. 
Agent-winning plays terminate in an accepting state, with no preceding states being accepting. And environment-winning plays are infinite plays where none of the states within the plays is accepting. As in \LTLf synthesis, we focus on the \DFA games where the agent moves first.

A play $\tau =s_0 s_1 \cdots$ is consistent with an agent strategy $g: (2^\X)^* \to 2^\Y$ if for every $0 \leq i < |\tau|$~($|\tau| = \infty$ if $\tau$ is infinite), there exists $X_i\in2^\X$ such that $s_{i+1}=\delta(s_i,X_i\cup Y_i)$, where $Y_0 = g(\epsilon)$ and $Y_{i+1} = g(X_0 X_1 \cdots X_i)$. An agent strategy $g$ is a winning strategy, referred to as agent-winning strategy, from state $s$ if there exists no environment-winning plays that are consistent with $g$ starting from $s$. A state $s$ is an agent-winning state if there exists an agent-winning strategy from $s$. A state is environment-winning iff it is not agent-winning. We denote by $\awin(\G)$ the set of agent-winning states in a \dfa game $\G$. For \DFA games, each agent strategy $g$ can also be represented as a positional strategy, as a function $\pi:S\to2^\Y$, which provides the decisions for the agent based on the current state. Conversely, an agent positional strategy $\pi$ induces an agent strategy $g$ as follows: $g(\epsilon) = \pi(s)$ and for every finite trace $\rho$, let $\tau$ be the run of $\G$ on $\rho$ (i.e., starting in state $s$) and define $g(\rho|_{\X}) = \pi(s')$, where $s'$ is the last state in $\tau$. For any \dfa game $\G$, there exists an agent strategy $\pi$ such that $\pi$ is agent-winning for every state $s\in\awin(\G)$. This strategy is called a uniform agent-winning strategy. Hereafter, we consider only uniform agent-winning strategies and omit the term `uniform'. The following theorem establishes the relation between \ltlf synthesis and \dfa game.


\begin{theorem}[\citep{GV15}]\label{thm:pre}
Given an \LTLf synthesis specification $(\varphi,\X,\Y)$ and a \dfa ${\G_{\varphi}}$ such that $\LL(\varphi)=\LL(\G_\varphi)$, $(\varphi,\X,\Y)$ is realizable iff the initial state of $\G_\varphi$ is an agent-winning state in the \dfa game $\G_\varphi$.
\end{theorem}

\section{Theoretical Foundations}\label{sec:foundation}
We present in this section the theoretical foundations of our compositional forward \ltlf synthesis framework
. The crux of our compositional \LTLf synthesis approach is to bypass the direct synthesis of the complete \LTLf formula $\varphi = \bigwedge_{1\leq i\leq n}\varphi_i$. Instead, following the typical divide-and-conquer compositional principle, we synthesize each conjunct $\varphi_i$ before composing the results. This method allows us to utilize a structured synthesis procedure that leverages intermediate results to simplify subsequent composition steps, thereby diminishing the overall synthesis difficulty.

For an \LTLf specification $(\varphi,\X,\Y)$ with $\varphi = \bigwedge_{1\leq i\leq n}\varphi_i$, it is straightforward to decompose $\varphi$ into conjuncts $\varphi_1, \cdots, \varphi_n$. The real challenge lies in composing the results of reasoning on the sub-specifications $(\varphi_i,\X,\Y)$ to derive the synthesis result for $(\varphi,\X,\Y)$. To this end, we begin by considering a simple case: if any sub-specification $(\varphi_i,\X,\Y)$ is found to be unrealizable, we can directly conclude that the original specification $(\varphi,\X,\Y)$ is unrealizable.

\begin{theorem}\label{thm:sub-unrea}
The \ltlf specification $(\bigwedge_{1\leq i\leq n}\varphi_i,\X,\Y)$ is unrealizable if there exists $1\leq i\leq n$ such that $(\varphi_i,\X,\Y)$ is unrealizable.
\end{theorem}
\begin{proof}
We prove the theorem by contradiction. Assume that $(\bigwedge_{1\leq i\leq n}\varphi_i,\X,\Y)$ is realizable. By Definition~\ref{def:realizability}, there exists a winning strategy $g: (2^\X)^* \to 2^\Y$ such that for an arbitrary infinite sequence $\lambda = X_0 X_1 \cdots \in (2^\X)^\omega$, there is $k \geq 0$ such that $\rho\models\bigwedge_{1\leq i\leq n}\varphi_i$ holds, where $\rho=(X_0\cup g(\epsilon)) (X_1\cup g(X_0)) \cdots (X_k\cup g(X_0,\cdots,X_{k-1}))$. Then we have $\rho\models\varphi_i$ for $1\leq i\leq n$, which indicates that $(\varphi_i,\X,\Y)$ is realizable for $1\leq i\leq n$. This contradicts the condition that there exists $1\leq i\leq n$ such that $(\varphi_i,\X,\Y)$ is unrealizable. 
\end{proof}

The synthesis problem becomes more challenging when all sub-specifications are realizable, making it necessary to apply appropriate composition operations. Different from previous works where the composition is performed during \DFA construction, our approach executes the composition operation when each sub-specification has been determined to be realizable and a corresponding strategy is synthesized. Our compositional method performs the composition on the strategies associated with each sub-specification, rather than on the complete \DFAs of sub-specifications.

To formulate the composition of strategies, we utilize the \emph{agent-winning region} in \DFA games. Intuitively, the agent-winning region, consisting of all agent-winning states, captures all possible ways the agent can win the game, regardless of how the environment behaves. Consequently, it represents all the information required from each sub-specification to realize the original specification. Notably, the environment-winning states, which contribute no useful information for the agent, are excluded from these regions and merged into a single state, hence leading to a reduced state space. Therefore, later in the composition procedure, we only need to make use of these reduced \DFAs instead of the complete \DFAs of the sub-specifications, improving the efficiency of subsequent synthesis steps.

\begin{definition}[Agent-Winning Region]\label{def:aequiv}
Given a \dfa game on $\G=(2^{\X\cup\Y},init,S,\delta,F)$ with the set of agent-winning states $\awin(\G)\subseteq S$, the corresponding agent-winning region is represented by the \dfa 
\begin{equation}\label{eq:awr}
    \awr(\G)=(2^{\X\cup\Y},init,\awin(\G)\cup\{ew\},\delta',F)\text{ ,}
\end{equation}
where $ew$ is a special state representing the set of environment-winning states in $\G$. The transition function $\delta'$ is defined as follows:
\begin{flalign}\label{eq:awr-delta}
    &\delta'(s,X\cup Y) = \nonumber&\\
    &\quad\begin{cases}
        \delta(s,X\cup Y)&\parbox[t]{4.7cm}{if $s\neq ew$ and for every $X'\in2^\X$, $\delta(s,X'\cup Y)\in \awin(\G)$;}\\
        ew&\text{otherwise.}
    \end{cases}&
\end{flalign}
\end{definition}

We next show that $\awr(\G)$ does not exclude any information that the agent needs to win the original game $\G$.

\begin{definition}[Agent-Equivalent \DFA Games]
Let $\G$ be a \DFA game and $\G'$ a pruning of $\G$. $\G$ and $\G'$ are agent-equivalent (denoted $\G\aequiv\G'$) iff every agent-winning strategy in the \DFA game $\G$ is an agent-winning strategy in the \DFA game $\G'$, and vice-versa.
\end{definition}

\begin{lemma}\label{lem:win-equiv}
    Given a \DFA game $\G$ and its agent-winning region $\awr(\G)$, $\G$ and $\awr(\G)$ are agent-equivalent.
\end{lemma}
\begin{proof}
    We begin by showing that every agent-winning strategy in the \DFA game $\G=(2^{\X\cup\Y},S,\delta,F)$ is also an agent-winning strategy in the reduced \DFA game $\awr(\G)=(2^{\X\cup\Y},\awin(\G)\cup\{ew\},\delta',F)$. 

    
    Let $g: (2^\X)^\star \rightarrow 2^\Y$ be a strategy. Let $\tau_g$ be a play induced by $g$ on the game $\G$, such that for every $0 \leq i < |\tau_g|$, we have that $s_{i+1}=\delta(s_i,X_i\cup Y_i)$ holds for some $X_i\in2^\X$, where $Y_0 = g(\epsilon)$ and $Y_{i+1} = g(X_0 X_1 \cdots X_i)$. We will show that when $g$ is an agent-winning strategy, no plays induced by $g$ on $\G$ visits a state $s \notin \awin(\G)$. 
    
    By means of contradiction, suppose $g$ is an agent-winning strategy such that there exists a play $\tau_g$ visits $s_i \notin \awin(\G)$ for some $i \geq 0$. Since \DFA games are determined games and both players have uniform winning strategies, the environment can begin executing an environment-winning strategy from $s_i \notin \awin(\G)$. Then, by definition of environment-winning strategies, we have that every resulting play will never visit $F$, i.e., is an environment-winning play. Thus, we have a contradiction. 
    
    By construction, i.e., Equation~(\ref{eq:awr}), since $\awr(\G)$ is defined over the winning region of $\G$, every agent-winning strategy $g$ in the \DFA game $\G$ can be executed in $\awr(\G)$. Therefore, the plays produced by $g$ on $\G$ are the same as those produced by $g$ on $\awr(\G)$. In other words, all the plays produced by $g$ on $\awr(\G)$ are agent-winning plays. Thus, $g$ is also an agent-winning strategy on $\awr(\G)$.
      
    We now show that a strategy $g$ that is not agent-winning in $\G$ is not agent-winning in $\awr(\G)$. If $g$ is not agent-winning, there exists a play $\tau_g$ produced by $g$ that is an environment-winning play. Therefore, either $\tau_g$ stays within $\awin(\G)$ but never visits $F$ or $\tau_g$ visits a state $s_i \not\in \awin(\G)$ for some $i \geq 0$. In the former case, Equation~(\ref{eq:awr}) assures that $\tau_g$ also exists when playing $g$ on $\awr(\G)$ such that $g$ is not an agent-winning strategy on $\awr(\G)$. In the latter case, there exists a new play $\tau'_g$ that shares the same prefix $h = s_0 s_1 \cdots s_{i-1}$ before $\tau_g$ visits $s_i \not\in \awin(\G)$ ($i$ indicates the first occurrence of such state), and $\tau'_g = h \cdot (ew)^\omega$~(note that $ew$ represents all the environment-winning states). $\tau'_g$ is indeed an environment-winning play, hence $g$ is not an agent-winning strategy on $\awr(\G)$.
\end{proof}

The agent-equivalent relation is preserved under \DFA product and minimization.
\begin{lemma}\label{lem:euiv-product}
Let $\G_1$, $\G_1'$, $\G_2$, and $\G_2'$ be \DFA games such that $\G_1\aequiv\G_1'$ and $\G_2\aequiv\G_2'$. We have $\G_1\times \G_2\aequiv \G_1'\times \G_2'$. 
\end{lemma}
\begin{proof}
    Clearly, $g$ is an agent-winning strategy for $\G_1\times \G_2$ iff $g$ is an agent-winning strategy for both $\G_1$ and $\G_2$, as $\LL(\G_1\times \G_2) = \LL(\G_1) \cap \LL(\G_2)$ . Then by Lemma~\ref{lem:win-equiv}, we have that $g$ is an agent-winning strategy for both $\G'_1$ and $\G'_2$. Therefore, $g$ is an agent-winning strategy for $\G_1'\times \G_2'$.
\end{proof}

\begin{lemma}\label{lem:euiv-minimize}
Let $\G$ and $\G_m$ be \DFA games such that $\G_m$ is the minimal \DFA of $\G$. We have $\G \aequiv \G_m$. 
\end{lemma}
\begin{proof}
$g$ is an agent-winning strategy for $\G$ iff every play induced by $g$ terminates at an accepting state of $\G$, hence the corresponding traces are accepted by $\D$ and consequently by $\G_m$ as $\LL(\G) = \LL(\G_m)$. Therefore, $g$ is also an agent-winning strategy for $\G_m$.
\end{proof}

At this point, we can derive the theorem for solving the \ltlf synthesis problem by compositions on agent-winning regions w.r.t. sub-specifications.

\begin{theorem}\label{thm:main}
For an \ltlf specification $(\bigwedge_{1\leq i\leq n}\varphi_i,\X,\Y)$, where  $(\varphi_i,\X,\Y)$ is realizable for all $1\leq i\leq n$, and let $\G_i$ be \dfa of $\varphi_i$, the complete specification $(\bigwedge_{1\leq i\leq n}\varphi_i,\X,\Y)$ is realizable iff the initial state of $\awr(\G_1\times\cdots\times\G_n)$ is an agent-winning state.
\end{theorem}
\begin{proof}
By Lemma~\ref{lem:win-equiv}, we have that $\G_1\times\cdots\times\G_n$ and $\awr(\G_1\times\cdots\times\G_n)$ are agent-equivalent. Then we have the following:

$(\bigwedge_{1\leq i\leq n}\varphi_i,\X,\Y)$ is realizable.
$\xLeftrightarrow{\text{Theorem~\ref{thm:pre}}}$
The initial state of $\G_\varphi$ is an agent-winning state with $\varphi=\bigwedge_{1\leq i\leq n}\varphi_i$.
$\xLeftrightarrow{\text{Lemma~\ref{lem:and2product}}}$
The initial state of $\G_1\times\cdots\times\G_n$ is an agent-winning state.
$\xLeftrightarrow{\text{Lemmas~\ref{lem:win-equiv}\&\ref{lem:euiv-product}\&\ref{lem:euiv-minimize}}}$
The initial state of $\awr(\G_1\times\cdots\times\G_n)$ is agent-winning.
\end{proof}

\section{Compositional Techniques}\label{sec:imple}
In this section, we introduce in detail our compositional on-the-fly \LTLf synthesis technique and its implementation.

\subsection{Main Framework}
The main idea of the compositional synthesis framework lies in decomposing a given \LTLf specification into smaller sub-specifications, processing them and ultimately composing the intermediate results. Algorithm~\ref{alg:main} outlines the main structure of this approach. It takes an \ltlf specification $(\varphi_1 \wedge \cdots \wedge \varphi_n,\X,\Y)$ as input and returns either a non-empty strategy if realizable or an empty set if unrealizable. The algorithm starts with a fast unrealizability check (Line~\ref{line:main:unreaCheck}) by checking whether any sub-specification is unrealizable. This is done by synthesizing separately each sub-specification using standard on-the-fly synthesis techniques. If an unrealizable sub-specification is found, the entire specification concludes to be unrealizable directly~(Theorem~\ref{thm:sub-unrea}).

If the fast unrealizablity check does not yield a result, the algorithm proceeds to the main composition procedure, which iteratively processes each sub-specification in the second \textit{for}-loop. As detailed in Section~\ref{sec:foundation}, the composition of sub-specifications requires combining only the agent-winning regions of the \DFA games corresponding to these specifications. Therefore, a \DFA, $awr\_\G$, is initialized before entering the loop to track the composed agent-winning regions of all processed sub-specifications. Initially, $awr\_\G$ is set to the \dfa $\G_{\ttrue}$, which accepts any arbitrary trace. Within the loop, each iteration processes a new sub-specification and composes it with the previously processed ones. Specifically, the $i$-th~($1 \leq i \leq n$) iteration corresponds to processing the sub-specification $\varphi_i$ and composing it with $(\varphi_1\wedge\cdots\wedge\varphi_{i-1},\X,\Y)$. At Line~\ref{line:main:compose}, the function $\Compose()$ performs the composition and updates $awr\_\G$ as $\awr(\G_{\varphi_1}\times\cdots\times\G_{\varphi_i})$. Specifically, the composition operation implemented in the function $\Compose()$ at Line~\ref{line:main:compose} satisfies the following:
\begin{flalign}\label{eq:compose}
  &\Compose(\awr(\G_{\psi_j}),(\psi_k,\X,\Y)) 
  =\nonumber&\\
  &\begin{cases}
      \awr(\G_{\psi_j}\times\G_{\psi_k})&\text{if $(\psi_j\wedge\psi_k,\X,\Y)$ is realizable;}\\
      \text{Null}&\text{otherwise.}
  \end{cases}&
\end{flalign}

During each iteration, if $\Compose()$ returns `Null', indicating that $(\varphi_1\wedge\cdots\wedge\varphi_i,\X,\Y)$ is unrealizable, the original specification $(\varphi, \X, \Y)$ is concluded to be unrealizable~(Theorem~\ref{thm:sub-unrea}). Once the loop completes and all intermediate results are composed, the algorithm determines that $(\varphi, \X, \Y)$ is realizable, and returns an agent strategy. This agent strategy is built from the agent-winning region $awr\_\G$, which is intuitive and encapsulated into an API $\BuildStrategy()$. In fact, during the final (i.e., $n$-th) iteration, no further composition operations will be performed. Therefore, it is not required to compute the agent-winning region. Instead, in the final composition step, we only need to check the realizability and derive an agent-winning strategy. For conciseness, this is described here but is not explicitly written out in the algorithm.

\begin{algorithm}[t]
\LinesNumbered
\DontPrintSemicolon
\caption{Compositional Synthesis}
\label{alg:main}
\KwIn{An \ltlf specification $(\varphi_1 \wedge \cdots \wedge \varphi_n,\X,\Y)$}
\KwOut{Agent strategy if the specification is realizable;\\\qquad\quad~~~$\emptyset$ otherwise.} 
\For{$i=1\cdots n$}
{\label{line:main:for1begin}
  \If{$(\varphi_i,\X,\Y)$ \textup{is unrealizable}}
  {\label{line:main:unreaCheck}
    \KwRet $\emptyset$\;
  }
}\label{line:main:for1end}
$awr\_\G\coloneqq\G_\ttrue$\quad\textcolor{blue}{$\rhd$ $\G_{\ttrue}$ is a \dfa such that $\LL(\G_{\ttrue})=\LL(\ttrue)$.}\;\label{line:main:for2begin}
\For{$i=1\cdots n$}
{
  $awr\_\G\coloneqq \Compose(awr\_\G,(\varphi_i,\X,\Y))$\;\label{line:main:compose}
  \If{$awr\_\G = \,$\textup{Null}}
  {
    \KwRet $\emptyset$\;
  }
    \textcolor{blue}{$\rhd$ Loop Invariant: $awr\_\G = \awr(\G_{\varphi_1}\times\cdots\times\G_{\varphi_i})$.}\;
}\label{line:main:for2end}
\KwRet $\BuildStrategy(awr\_\G)$\;
\end{algorithm}

We now introduce two composition variants. The first variant, \emph{individual composition}, focuses on fully leveraging minimization to reduce the state space during composition. The second variant, \emph{incremental composition}, aims to incorporate on-the-fly synthesis to guide the composition process.

\subsection{Individual Composition} Intuitively, this approach solves each sub-specification independently, unaffected by the composition. Specifically, this variant first computes the agent-winning region of each new sub-specification as a DFA and minimizes it to reduce the state space, and then integrates it into the composition. As outlined in Algorithm~\ref{alg:individual}, the procedure first computes the agent-winning region $\awr(\G_{\psi_2})$ of the sub-specification $(\psi_2,\X,\Y)$ to be composed. It then performs a \dfa product to compute $\awr(\G_{\psi_1})\times\awr(\G_{\psi_2})$ (Line \ref{line:individual:product}), and solves the resulting \dfa game (Line \ref{line:individual:solveGame}). If the initial state of $\awr(\G_{\psi_1})\times\awr(\G_{\psi_2})$ is agent-winning, the algorithm returns $\awr(\G_{\psi_1}\times\G_{\psi_2})$; otherwise, it returns `Null'. Note that whenever an agent-winning region is constructed, it is minimized to reduce the state space of subsequent steps.


The computation of $\awr(\G_{\psi_2})$ is implemented using the API $\GetAwr()$ (Line \ref{line:individual:getAwr}), which adapts existing \LTLf synthesis approaches. For a given \DFA game, the key to building the agent-winning region is computing the set of agent-winning states. When applying the on-the-fly synthesis approach to compute the agent-winning region, the forward search must not terminate as soon as the current state is identified as agent-winning. Instead, it should explore all possible agent choices and continue the search recursively. 

\begin{lemma}\label{lem:algo-individual}
The implementation of $\Compose()$ in Algorithm~\ref{alg:individual} satisfies Equation~(\ref{eq:compose}).
\end{lemma}
\begin{proof}
By Lemmas~\ref{lem:win-equiv} and \ref{lem:euiv-product}, we have that $\G_{\psi_1}\times\G_{\psi_2} \aequiv \awr(\G_{\psi_1})\times\awr(\G_{\psi_2}) \aequiv \awr(\G_{\psi_1}\times\G_{\psi_2})$.
If $(\psi_1\wedge\psi_2,\X,\Y)$ is realizable, the initial state of $\awr(\G_{\psi_1})\times\awr(\G_{\psi_2})$ is agent-winning (by Theorems~\ref{thm:pre} and \ref{thm:main}). In this case, $\awr(\G_{\psi_1}\times\G_{\psi_2})$ is returned. Conversely, if $(\psi_1\wedge\psi_2,\X,\Y)$ is unrealizable, the initial state of $\awr(\G_{\psi_1})\times\awr(\G_{\psi_2})$ is not agent-winning (by Theorems~\ref{thm:pre} and \ref{thm:main}), and `Null' is returned.  
\end{proof}

\begin{algorithm}[t]
\LinesNumbered
\DontPrintSemicolon
\caption{$\protect\Compose()$ - Individual}
\label{alg:individual}
\KwIn{A \dfa game $\awr(\G_{\psi_1})$ and an \ltlf specification $(\psi_2,\X,\Y)$}
\KwOut{$\awr(\G_{\psi_1}\times\G_{\psi_2})$ if $(\psi_1\wedge\psi_2,\X,\Y)$ is realizable; Null otherwise.}

$\awr(\G_{\psi_2})\coloneqq\GetAwr(\psi_2,\X,\Y)$\quad\textcolor{blue}{$\rhd$ Minimize.}\;\label{line:individual:getAwr}
$\G_{tmp}\coloneqq\awr(\G_{\psi_1})\times\awr(\G_{\psi_2})$\;\label{line:individual:product}
$\SolveDfaGame(\G_{tmp})$\;\label{line:individual:solveGame}
\If{\textup{the initial state of $\G_{tmp}$ is agent-winning}}
{
  \KwRet $\awr(\G_{\psi_1}\times\G_{\psi_2})$\quad\textcolor{blue}{$\rhd$ Minimize.}\;
}
\Else
{
  \KwRet Null\;
}
\end{algorithm}

This variant fully leverages the benefits of minimization; however, it requires precomputing the agent-winning regions of all sub-specifications prior to minimization. Since this process is 2EXPTIME-complete in the worst case, the computational cost can become prohibitive for large specifications, despite the agent-winning regions being constructed on the fly. To navigate this complexity, we propose a second variant that incorporates on-the-fly synthesis into the composition. 

\subsection{Incremental Composition} Instead of precomputing the agent-winning region of a new sub-specification before composition, this approach utilizes the previously computed agent-winning region of those composed sub-specifications to guide the construction of the agent-winning region of the new sub-specification during composition. Hence, the composition is conducted in an incremental manner. 


The incremental composition is outlined in Algorithm \ref{alg:incremental}. This algorithm essentially solves the \dfa game $\awr(\G_{\psi_1})\times\G_{\psi_2}$ using an on-the-fly synthesis technique, leveraging the precomputed agent-winning region $\awr(\G_{\psi_1})$ to guide the search over $\G_{\psi_2}$ effectively to guide the search, meanwhile composing them. Algorithm \ref{alg:incremental} is a variant of the existing on-the-fly \ltlf synthesis approach (cf. \citep{XLZSLPV24}). It traverses the state space of $\awr(\G_{\psi_1})\times\G_{\psi_2}$ in a depth-first manner, during which all agent-winning states within $\awr(\G_{\psi_1})\times\G_{\psi_2}$ are identified and the corresponding agent-winning region is built. The key difference from the standard on-the-fly synthesis search is introduced at Line~\ref{line:incremental:ewinY}. With the aid of $awr\_\G_1$, i.e, $\awr(\G_{\psi_1})$, the search space is pruned by eliminating agent choices that allow the environment to win, captured in the set $\EwinAgentChoices(s_1, awr\_\G_1)$. Specifically, given $awr\_\G_1=(2^{\X\cup\Y},\awin(\G_1)\cup\{ew\},\delta_1',F_1)$ and $s_1\in\awin(\G_1)\cup\{ew\}$, we have $\EwinAgentChoices(s_1, awr\_\G_1) =\{Y\in2^\Y\mid\exists X\in2^\X.\delta_1'(s_1,X\cup Y)=ew\}$.
Additionally, the composition of \DFA states occurs at Lines \ref{line:incremental:readSuccS1}-\ref{line:incremental:stateCompose}. The state $s_1'$ is retrieved from the precomputed $awr\_\G_1$ (i.e., $\awr(\G_{\psi_1})$), while the state $s_2'$ of $\G_{\psi_2}$ is computed on the fly at the moment. 

\begin{algorithm}[!t]
\LinesNumbered
\DontPrintSemicolon
\caption{$\protect\Compose()$ - Incremental}
\label{alg:incremental}
\KwIn{A \dfa $\awr(\G_{\psi_1})$ and an \ltlf specification $(\psi_2,\X,\Y)$}
\KwOut{$\awr(\G_{\psi_1}\times\G_{\psi_2})$ if $(\psi_1\wedge\psi_2,\X,\Y)$ is realizable; Null otherwise.}

$awin\_state$, $ewin\_state$, $undetermined\_state\coloneqq\emptyset$\;
$\DFSearch((\psi_1,\psi_2),\awr(\G_{\psi_1}))$\;
\If{$(\psi_1,\psi_2)\in awin\_state$}
{
  \KwRet $\BuildAwr()$\quad\textcolor{blue}{$\rhd$ Minimize.}\;
}
\Else
{
  \KwRet Null\;
}
\;
\myproc{$\DFSearch((s_1,s_2),awr\_\G_1)$}
{
  \If{$(s_1,s_2)\in awin\_state\cup ewin\_state\;\cup$ $~~~~~~~~~~~~~~~~~~~~ undetermined\_state$}
  {
    \KwRet\;
  }
  \If{$\IsAccepting((s_1,s_2))$}
  {
    $awin\_state.insert(s_1,s_2)$\;
  }
  $undetermined\_state.insert((s_1,s_2))$\;
  $ewin\_for\_all\_Y\coloneqq\trueVal$\;
  \For{$Y\in (2^\Y\setminus\EwinAgentChoices(s_1, awr\_\G_1))$}
  {\label{line:incremental:ewinY}
    $ewin\_for\_some\_X\coloneqq\falseVal$\;
    $undetermined\_for\_some\_X\coloneqq\falseVal$\;
    \For{$X\in2^\X$}
    {
      $s_1'\coloneqq \ReadSuccessor(awr\_\G_1,s_1,X\cup Y)$\;\label{line:incremental:readSuccS1}
      $s_2'\coloneqq \GetSuccessor(s_2,X\cup Y)$\;
      $\DFSearch((s_1',s_2'),awr\_\G_1)$\;\label{line:incremental:stateCompose}
      \If{$(s_1',s_2')\in ewin\_state$}
      {
        $ewin\_for\_some\_X\coloneqq\trueVal$\;
        \Break\;
      }
      \ElseIf{$(s_1',s_2')\in undetermined\_state$}
      {
        $undetermined\_for\_some\_X\coloneqq\trueVal$\;
        \Continue\;
      }
    }
    \If{$\neg ewin\_for\_some\_X$}
    {
      $ewin\_for\_all\_Y\coloneqq\falseVal$\;
      \If{$\neg undetermined\_for\_some\_X$}
      {
        $undetermined\_state.remove((s_1,s_2))$\;
        $awin\_state.insert((s_1,s_2))$\;
      }
    }
  }
  \If{$ewin\_for\_all\_Y$\textup{ and }$(s_1,s_2)\notin awin\_state$}
  {
    $undetermined\_state.remove((s_1,s_2))$\;
    $ewin\_state.insert((s_1,s_2))$\;
  }
  \If{$\IsSccRoot((s_1,s_2))$}
  {
    $scc\coloneqq\GetScc((s_1,s_2))$\;
    $\BackwardSearch(scc)$\;
  }
}

\end{algorithm}

\begin{lemma}\label{lem:algo-incremental}
The implementation of $\Compose()$ in Algorithm~\ref{alg:incremental} satisfies Equation~(\ref{eq:compose}).
\end{lemma}
\begin{proof}
By Lemmas~\ref{lem:win-equiv} and \ref{lem:euiv-product}, we have that $\G_{\psi_1}\times\G_{\psi_2} \aequiv \awr(\G_{\psi_1})\times\G_{\psi_2} \aequiv \awr(\G_{\psi_1}\times\G_{\psi_2})$.
If $(\psi_1\wedge\psi_2,\X,\Y)$ is realizable, the initial state of $\awr(\G_{\psi_1})\times\G_{\psi_2}$ is agent-winning (by Theorems~\ref{thm:pre} and \ref{thm:main}). In this case, $\awr(\G_{\psi_1}\times\G_{\psi_2})$ is returned. Conversely, if $(\psi_1\wedge\psi_2,\X,\Y)$ is unrealizable, the initial state of $\awr(\G_{\psi_1})\times\awr(\G_{\psi_2})$ is not agent-winning (by Theorems~\ref{thm:pre} and \ref{thm:main}), and `Null' is returned.
\end{proof}
%
Comparing two composition variants reveals key differences in their approaches. In the individual composition~(Algorithm \ref{alg:individual}), search and composition are performed sequentially, with the synthesis search space being $\G_{\psi_2}$. In contrast, the incremental composition~(Algorithm~\ref{alg:incremental}) integrates search into the composition, executing them in parallel. Hence the search space is $\awr(\G_{\psi_1})\times\G_{\psi_2}$.

\begin{theorem}
For an \ltlf specification $(\bigwedge_{1\leq i\leq n}\varphi_i,\X,\Y)$,  $(\bigwedge_{1\leq i\leq n}\varphi_i,\X,\Y)$ is realizable iff Algorithm \ref{alg:main} returns a non-empty strategy.
\end{theorem}
\begin{proof}
With Lemmas~\ref{lem:algo-individual} and \ref{lem:algo-incremental}, we can prove the loop invariants $awr\_\G=\awr(\G_{\varphi_1}\times\cdots\times\G_{\varphi_i})$ for the second \textit{for}-loop by induction over the value of $i$. Then we have:

$(\bigwedge_{1\leq i\leq n}\varphi_i,\X,\Y)$ is realizable.
$\xLeftrightarrow{\text{Theorem~\ref{thm:sub-unrea}}}$
For every $i$ with $1\leq i\leq n$, $(\varphi_i,\X,\Y)$ and $(\bigwedge_{1\leq i\leq n}\varphi_i,\X,\Y)$ are realizable. 
$\Leftrightarrow$
Algorithm \ref{alg:main} does not terminate within the two \textit{for}-loops and it returns a non-empty strategy.
\end{proof}






\section{Experimental Evaluation}\label{sec:exp}
We implemented the compositional synthesis approach in a prototype tool called \tool~\citep{artifact} and compared it against state-of-the-art \ltlf synthesizers. This section presents experimental results demonstrating that our compositional \ltlf synthesis approach outperforms existing \ltlf synthesizers. Among the two composition variants, the incremental composition demonstrates better performance than the individual composition. Therefore, unless otherwise specified, the default setting in \tool is the incremental composition variant. 

\subsection{Setup}
\paragraph{Benchmarks.} We collected, in total, 3380 \LTLf synthesis instances from literature: 3200 \emph{Random} instances \citep{ZTLPV17,BLTV20,DF21,XLHXLPSV24}, 140 \emph{Two-Player-Games} instances---including 20 \emph{single-counter}, 20 \emph{double-counters}, and 100 \emph{Nim} \citep{TV19,BLTV20}, and 40 \emph{Patterns} instances \citep{XLZSPV21}. 

\paragraph{Baseline.} We evaluated the performance of our compositional approach by comparing \tool with the four leading \ltlf synthesis tools: \lisa~\citep{BLTV20}, \lydia~\citep{DF21}, \nike~\citep{Favorito23}, and \tople\citep{XLZSLPV24}. Among these, \lisa and \lydia represent state-of-the-art \ltlf synthesis tools that are based on the backward approach, while \nike and \tople implement the forward on-the-fly synthesis approach. The correctness of \tool is empirically validated by comparing its results with those of the baseline tools.

\paragraph{Running Platform and Resources.} The experiments were run on a CentOS 7.4 cluster, where each instance has exclusive access to a processor core of the Intel Xeon 6230 CPU (2.1 GHz), 10 GB of memory, and a 30-minute time limit.


\begin{figure*}[ht]
    \centering
    \hfill
    \subfloat[]{
        \centering
        \includegraphics[width=0.23\linewidth]{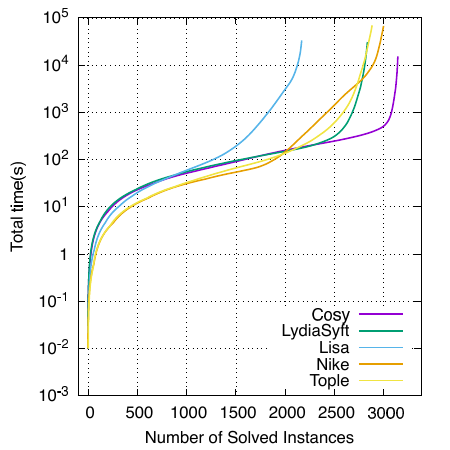}
    	\label{fig:speed}
    }
    \subfloat[]{
        \centering
        \includegraphics[width=0.23\linewidth]{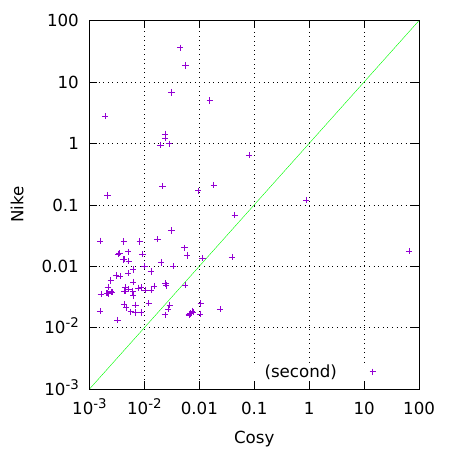}
        \label{fig:state}
    }
    \hfill
    \subfloat[]{
        \centering
       \includegraphics[width=0.23\linewidth]{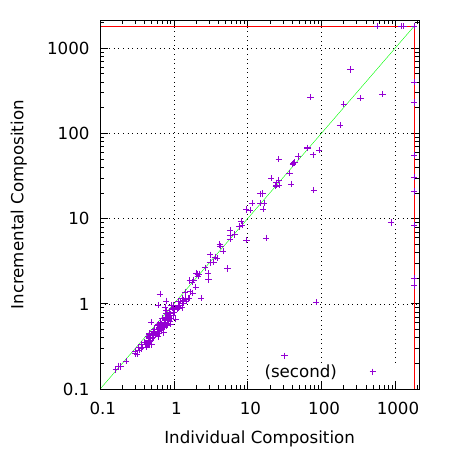}
	  \label{fig:time}
    }
    \hfill
    \subfloat[]{
        \centering
        \includegraphics[width=0.23\linewidth]{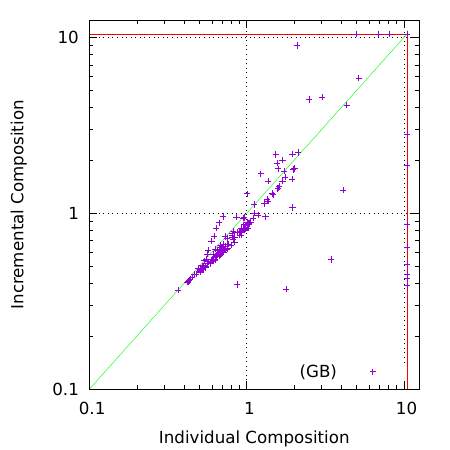}
    	\label{fig:memory}
    }
    \caption{Figure~\ref{fig:speed} - Comparison of cumulative solved instances across all evaluated tools over time. Figure~\ref{fig:state} - Comparison of the average time cost per \DFA state when \tool and \nike search the \DFA w.r.t. the original specification. Figures~\ref{fig:time} and \ref{fig:memory} - Comparison of the time and memory costs of individual and incremental compositions for instances involving composition operations, respectively. Points on the red line represent instances where one approach failed to solve.\\~}
\end{figure*}

\subsection{Results and Discussion}

\subsubsection{Comparison with Baseline} Out of a total of 3380 instances, \tool, \lydia, \lisa, \nike, and \tople successfully solve \textbf{3149}, 2835, 2169, 3000, and 2883 instances respectively, with \tool achieving the highest number of solved instances. Figure~\ref{fig:speed} shows that \tool in general outperforms other evaluated tools, solving a higher number of instances with less time cost. This result demonstrates the overall outperformance of our compositional approach.

Note that among all the 3380 instances, 1689 instances are in the form of conjunctions of \LTLf formulas, while the remaining 1691 cannot be decomposed into conjuncts. To specifically evaluate the effectiveness of our compositional approach, the subsequent comparison and discussion focuses only on the \emph{1689 decomposable instances}, excluding the non-decomposable ones. 

\begin{table}[ht]
\centering
\caption{Comparison of the number of solved instances across different tools for the 1689 decomposable instances.}\label{tab:overall}
\scriptsize
\renewcommand{\arraystretch}{1.6}
\begin{tabular}{ll|rrrrr}
\hline
                                             &              & \tool & \lydia & \lisa & \nike & \tople \\ \hline
\multicolumn{1}{p{0.9cm}|}{\multirow{2}{=}{\emph{Random\& Patterns}}} & Realizable   & 388    & 379     & 322    &\textbf{391}     & 364     \\
\multicolumn{1}{l|}{}                        & Unrealizable & \textbf{1145}    & 978     & 995    & 1051    & 921    \\ \cline{1-2}
\multicolumn{1}{p{0.9cm}|}{\multirow{3}{=}{\emph{Two-
Player-Games}}}    & \emph{s-counter}    & \textbf{12}   & \textbf{12}    & 7   & 5   & 4    \\
\multicolumn{1}{l|}{}                        & \emph{d-counters}   & \textbf{6}   & \textbf{6}    & 5   & 4   & \textbf{6}    \\
\multicolumn{1}{l|}{}                        & \emph{Nim}          & \textbf{39}   & 20    & 11   & 11   & 4    \\ \hline
\multicolumn{2}{l|}{Uniquely solved}                        & \textbf{75}   & 2    & 2   & 0   & 1    \\ \cline{1-2}
\multicolumn{2}{l|}{Total}                                  & \textbf{1590}   & 1395    & 1340   & 1462   & 1299    \\ \hline
\end{tabular}
\end{table}

Table~\ref{tab:overall} shows the number of decomposable instances solved by different tools. \tool, leveraging the compositional approach, demonstrates a significant advantage in solving capability. It achieves the highest number of uniquely solved instances (75) and the highest total number of solved instances~(1590), both numbers are substantially greater than those of the other tools. Across all benchmark families, \tool shows optimal performance, with the sole exception of a slight drop on \emph{Random\&Patterns-Realizable}, where \tool solves 388 instances, compared to the best 391. This is because, for realizable instances, \tool requires computing complete strategies for all sub-specifications, whereas the best-performing \nike operates on the fly over the original specification, potentially reducing the search space~(see below for more detailed analysis).

\subsubsection{Ablation Study} We now analyze the sources of performance improvement achieved through the compositional approach. For a given decomposable \ltlf specification $(\varphi_1\wedge\cdots\wedge\varphi_n,\X,\Y)$, there are three possible scenarios.

\textbf{\myi} There exists $i$ with $1\leq i\leq n$, such that $(\varphi_i,\X,\Y)$ is unrealizable. In this case, there exists an unrealizable single sub-specification, and no composition operation is performed. Among the 1590 solved decomposable instances, 1031 fall into this category.

\textbf{\myii} There exists $i$ with $1< i< n$ such that $(\varphi_1\wedge\cdots\wedge\varphi_i,\X,\Y)$ is unrealizable. In this case, a sub-specification composed from multiple single sub-specifications is determined to be unrealizable~(note that since all single sub-specifications are realizable, the composition operations are performed). Therefore, the state space w.r.t. the original specification is not searched. Among the solved decomposable instances, 69 belong to this category.

\textbf{\myiii} When neither of the above occurs, the state space w.r.t. the original specification is searched. Among the solved decomposable instances, 490 fall into this case. Figure~\ref{fig:state} compares the average time cost per \DFA state when the \DFA w.r.t. the original specification is searched between \tool and \nike. As depicted, a significantly larger number of points lie above the green reference line, indicating that our compositional approach effectively reduces the cost of subsequent searches by leveraging precomputed agent-winning regions.

\subsubsection{Comparing Individual and Incremental Compositions} To compare the individual and incremental composition variants, we focus on instances where composition operations are performed, corresponding to the latter two scenarios discussed above. For these instances, individual composition solves 554 cases, while incremental composition solves 559. Moreover, individual composition and incremental composition uniquely solve 3 and 8 instances, respectively, that the other approach fails to solve. Figures~\ref{fig:time} and \ref{fig:memory} compare the time and memory costs of the two variants, respectively. Both figures show similar distributions, with data points scattered on both sides of the green reference line. This suggests that each approach outperforms the other in certain instances, highlighting complementary strengths.


\section{Concluding Remarks}\label{sec:conclude}

We have presented a compositional on-the-fly approach to \ltlf synthesis, where composition operations are conducted at the synthesis level. An empirical comparison of our method with state-of-the-art \ltlf synthesizers shows that it achieves the best overall performance. Several future research directions are under consideration. First, the current implementation of \tool relies heavily on Binary Decision Diagrams ({\BDD}s), which requires exponential space relative to the number of variables, thereby limiting the capability of \tool. Exploring methods to reduce the usage and dependency on {\BDD}s could enhance our synthesizer. Second, while our compositional approach is currently limited to conjunctions, it would be interesting to extend the composition operation to support a richer set of operators.



\begin{ack}
We sincerely thank the anonymous reviewers for their valuable comments, which significantly improved the quality of this paper. This work was supported by the National Natural Science Foundation of China (Grants \#62372178 and \#U21B2015) and the Shanghai Collaborative Innovation Center of Trusted Industry Internet Software.
\end{ack}



\bibliography{cav,ok}
\clearpage
\appendix

\section{Computing Agent-Winning Region on the Fly}\label{app:awrOnTheFly}

Algorithm~\ref{alg:full-forward} presents the on-the-fly implementation of $\GetAwr()$, which identifies all the agent-winning states and builds $\awr(\G_\varphi)$ for realizable input specifications. The differences from the existing on-the-fly synthesis procedure are as follows.
\begin{itemize}
\item It recursively traverses every state that is not identified as environment-winning, even though a state that has been determined to be agent-winning.

\item All possible agent choices (i.e., $2^\Y$) are searched, even though the current state has been determined to be agent-winning.
\end{itemize}

\begin{algorithm}[!htbp]
\LinesNumbered
\DontPrintSemicolon
\caption{$\protect\GetAwr()$ by On-the-Fly Synthesis}\small
\label{alg:full-forward}
\KwIn{An \ltlf specification $(\varphi,\X,\Y)$}
\KwOut{$\awr(\G_{\varphi})$ if $(\varphi,\X,\Y)$ is realizable; Null otherwise.}

$awin\_state$, $ewin\_state$, $undetermined\_state\coloneqq\emptyset$\;
$\DFSearch(\varphi)$\;
\If{$\varphi \in awin\_state$}
{
  \textcolor{blue}{$\rhd$ Build $\awr(\G_\varphi)$ with $\awin(\G_\varphi)=awin\_state$.}\;
  \KwRet $\BuildAwr()$\;
}
\Else
{
  \KwRet Null\;
}
\;
\myproc{$\DFSearch(s)$}
{
  \If{$s\in awin\_state\cup ewin\_state\;\cup undetermined\_state$}
  {
    \KwRet\;
  }
  \If{$\IsAccepting(s)$}
  {
    $awin\_state.insert(s)$\;
  }
  $undetermined\_state.insert(s)$\;
  $ewin\_for\_all\_Y\coloneqq\trueVal$\;
  \For{$Y\in 2^\Y$}
  {\label{line:compose-on-the-fly:ewinY}
    $ewin\_for\_some\_X\coloneqq\falseVal$\; $undetermined\_for\_some\_X\coloneqq\falseVal$\;
    \For{$X\in2^\X$}
    {
      $s'\coloneqq \GetSuccessor(s,X\cup Y)$\;
      $\DFSearch(s')$\;
      \If{$s'\in ewin\_state$}
      {
        $ewin\_for\_some\_X\coloneqq\trueVal$\;
        \Break\;
      }
      \ElseIf{$s'\in undetermined\_state$}
      {
        $undetermined\_for\_some\_X\coloneqq\trueVal$\;
        \Continue\;
      }
    }
    \If{$\neg ewin\_for\_some\_X$}
    {
      $ewin\_for\_all\_Y\coloneqq\falseVal$\;
      \If{$\neg undetermined\_for\_some\_X$}
      {
        $undetermined\_state.remove(s)$\;
        $awin\_state.insert(s)$\;
        \textcolor{blue}{$\rhd$ No break!.}\;
      }
    }
  }
  \If{$ewin\_for\_all\_Y$}
  {
    $undetermined\_state.remove(s)$\;
    $ewin\_state.insert(s)$\;
  }
  \If{$\IsSccRoot(s)$}
  {
    $scc\coloneqq\GetScc(s)$\;
    $\BackwardSearch(scc)$\;
  }
}

\end{algorithm}

\end{document}